\documentclass[11pt]{amsart}

\usepackage[T1]{fontenc}
\usepackage[utf8]{inputenc}
\usepackage{bm}
\usepackage{bbm}

\usepackage{amsmath,amssymb,amsfonts,amsthm}
\usepackage{mathtools}

\usepackage[a4paper,margin=1in]{geometry}
\usepackage[colorlinks=true,
            linkcolor=blue,
            citecolor=blue,
            urlcolor=blue]{hyperref}

\usepackage{graphicx}
\usepackage{tikz}
\usepackage{float}

\usepackage{enumitem}

\newtheorem{theorem}{Theorem}[section]
\newtheorem{lemma}[theorem]{Lemma}

\theoremstyle{definition}

\theoremstyle{remark}

\title{Interpreting FCDNNs via RG on Exponential Family}

\author[F. Gong and Z. Xia]
{Fuzhou Gong$^{1,2}$ and Zigeng Xia$^{1}$}

\date{}

\begin{document}

\maketitle

\begingroup
\renewcommand\thefootnote{\arabic{footnote}}

\footnotetext[1]{Academy of Mathematics and Systems Science, Chinese Academy of Sciences,
Beijing 100190, China.}

\footnotetext[2]{School of Mathematics, Northwest University, Xi'an 710127, China.}

\footnotetext[3]{Email: fzgong@amt.ac.cn}

\footnotetext[4]{Email: xiazigeng@amss.ac.cn}

\endgroup

\begin{abstract}
We consider establishing the interpretability theory of deep learning through constructing a corresponding relationship between the renormalization group (RG) method in statistical physics and the training process of deep neural networks (DNNs). We have proved the constructed relationship using the one-dimensional Ising model as the input data. In this paper we generalize our results to the case of continuous input data, which is a necessary preparation for applying the corresponding framework to real-world data. To be representative, we consider a class of data distribution in the exponential family. We prove that when the parameters of fully connected (FC) DNNs achieve their optimal value after training, the characteristic parameters of the feature layer output of DNNs are equal to the fixed points of the characteristic parameters of input data under RG method for continuous fields. This conclusion shows that the training process of DNNs is equivalent to RG calculation on this kind of data and therefore the network can extract main features from the input data just like RG. Also, the equivalence further validates the correspondence framework we have established, providing an explanation for the outstanding performance of DNNs on real-world data.
\end{abstract}


\section{Introduction} \label{section1}

\subsection{Problem description and method framework}
The focus of this paper is the interpretability of fully connected deep neural networks (FCDNNs). In \cite{ref13, ref14}, we attempted to establish a theoretical framework for the interpretability of DNNs. From a theoretical perspective, an important reason for the success of DNNs is their ability to effectively extract the main features from large-scale datasets and, based on this ability, produce the target output for machine learning tasks. Therefore, drawing on the intuitive similarity between the compression of data in the forward propagation process of DNNs and the idea of transitioning from microscopic descriptions to macroscopic properties in statistical physics, we adopt the renormalization group (RG) method as the theoretical foundation of the framework. By leveraging the self-similarity exhibited by systems near critical points, the RG method applies coarse-graining or scale transformations (known as renormalization transformations) to statistical physical systems composed of numerous microscopic units and deduces the macroscopic features of the system. In this paper, we extend the approach of \cite{ref13, ref14}, shifting from the discrete Ising model to a more practical scenario involving continuous data distributions. Specifically, we consider cases where the data follows a class of exponential family distributions, the network architecture is a FCDNN, and the training algorithm employs simulated annealing.

For detailed information on interpretability and the RG method we adopted, please refer to reference \cite{ref13}. In \cite{ref13} we constructed the correspondence relationship between DNNs and RG and proved the equivalence between them on the one-dimensional Ising model on the finite lattice. For detailed information on the framework we established, as well as an explanation of the motivations behind our method, please refer to \cite{ref13}. Here, we briefly outline the steps of the method. The following figure shows the conception of our approach:
\begin{figure}[H]
	\centering
	\includegraphics[width=0.9\textwidth]{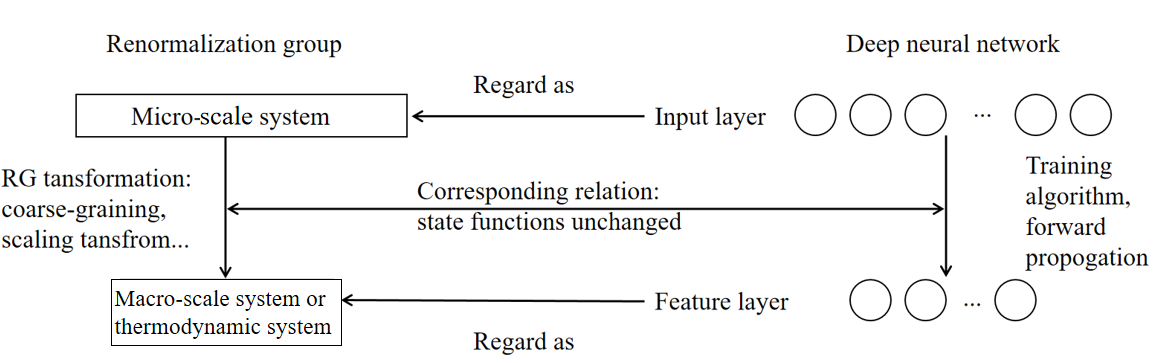}\\
	\caption{Framework of our method}\label{structure}
\end{figure}
Our method is composed of three steps:

Step 1. The input dataset of a DNN is regarded as a micro-scale statistical physical system. It can be analysed by the RG method to obtain a macro-scale system or a thermodynamic system with the characteristic parameters achieving their fixed points. This is called the canonical RG process.

Step 2. The main features extracted by the DNN (usually they are the results in the penultimate layer before final fully connected layer, named feature layer) are regarded as a macro-scale system or a thermodynamic system. The Hamiltonian and partition function of this system can be defined similarly to the original input data system in Step 1. We consider that the network structure and the training process accomplish the renormalization of the input data system. We call it non-canonical RG and restrict it by the same criteria as the canonical RG. Like the case of canonical RG, the main features extracted by the DNN exhibit self-similarity and therefore are dimensionless in terms of dimensional analysis.

Step 3. We prove that the two RG processes actually extract the same macroscopic features by comparing the feature layer output system of DNN and the macro-scale system in RG, showing that their characteristic parameters have the same limit, which corresponds to the critical point of the model.

The RG method is an effective method in statistical physics for deriving macroscopic features from a system's microscopic description. The training process of a DNN is equivalent to RG in terms of its operation on specific data distributions, which explains why DNNs can successfully extract the main features from data. The knowledge learned by a DNN corresponds to the macroscopic features of the system in the physical sense, representing the overall properties of the data. The network utilizes these features in the layer immediately preceding the output layer to arrive at its final decision.

The main conclusion we obtained in \cite{ref13} is that, when the one-dimensional Ising model on finite lattice is used as the input data for a FCDNN, during the simulated annealing training process of the network, as the network parameters converge to the optimal values in probability, the coupling constant in the partition function of the physical system defined by the outputs of the corresponding feature layers converges to zero, which is consistent with the stable fixed point of the coupling constant of the one-dimensional Ising model obtained by the real-space RG method.

Compared with \cite{ref13}, the work in this paper represents a necessary extension and deepening of the framework. The real-space RG of the one-dimensional Ising model is theoretically clear and rigorous, so in \cite{ref13} we only need to discuss the definition of relevant physical quantities in the non-canonical RG process of DNNs. The establishment of conclusions such as Theorem 4.1 in \cite{ref13} helps demonstrate more clearly the rationality of the framework we constructed, laying a methodological foundation for subsequent handling of more complex data distributions. In practical applications, the training data of DNNs follow continuous distributions, making it necessary to extend this research direction.

\subsection{Physical principles and research objectives}

When we consider continuous datasets, it is natural to start with the representative exponential family of distributions in statistics. The exponential family provides a unified framework that encompasses many important probability distributions and covers a wide range of common data distributions encountered in practice. Therefore, considering the exponential family does not lead to a loss of generality and holds practical significance. 

The difference between continuous distribution data and the Ising model case is that, in the real-space RG method for discrete physical models, the renormalization transformation is a coarse-graining process, which is carried out by combining adjacent microscopic units into an equivalent block. For continuous data, coarse-graining operations cannot be performed. Instead, a continuously varying scale parameter $\lambda$ is introduced, and the data system is treated as a continuous field, utilizing the RG method based on scale transformation. The form of the exponential family can naturally be viewed as a Gibbs distribution, which allows us to directly apply the RG method. Also, in the RG method for continuous fields, the Hamiltonian of the system, i.e., the energy function, must satisfy the Landau-Ginzburg-Wilson condition \cite{ref54}. The exponential family distribution data we consider satisfies this condition.

Due to the differences in specific methods, for exponential family distributions, we need to reconsider the implementation approaches in the 3 steps mentioned above. We introduce step by step the physical principles followed in establishing the correspondence relationship in this paper, particularly the parts that differ from \cite{ref13}.

First, in Step 1, we need to describe how the Hamiltonian is defined when the data of an exponential family distribution serves as the microscopic representation of a statistical physical system. When defining this physical system, we refer to the canonical ensemble (see \cite{ref53} for related definitions and details). The canonical ensemble is a statistical ensemble that describes a closed system in thermal equilibrium with its surroundings. In such a system, the number of particles and the temperature remain constant, while energy exchange with the environment is allowed. It is characterized by the probability distribution of the system being in a given microscopic state based on its energy. Since the exponential family distribution naturally takes the form of the Gibbs distribution in the canonical ensemble, the form of its Hamiltonian $\mathcal{H}$ is determined. The Hamiltonian represents the total energy of the system in a specific microscopic state. Therefore here it represents the energy of a given data point, with the Boltzmann factor $\mathrm{e}^{-\mathcal{H}}$ proportional to the distribution density function of the system in that microscopic state. In the RG method for continuous fields, the system under consideration is described by a continuous field $\sigma(x),\ x{\in}\mathbb{R}^N$, and the form of its probability distribution is the same as that of the canonical ensemble distribution. When the distribution density function of the data is given, its support is in $\mathbb{R}^N$, and for a specific dataset, the values of the Hamiltonian should be concentrated on the data sample points. Therefore, in the case we consider, the field corresponding to the data system is vector-valued and defined as ${\sigma}_{i}(x)=x_i, i=1,2,\cdots,N.$ In the canonical ensemble, the partition function of the system is defined as the sum of Boltzmann factors over all possible microscopic states. In the continuous field model, this corresponds to an integral over all possible field configurations in the space of continuous functions, whereas in the case of network training data, it correspondingly becomes an integral over $\mathbb{R}^N$ and equals the normalization factor of the distribution. Based on the above principles, exponential family data can be regarded as a physical system to which the RG method for continuous fields can be directly applied.

In Step 2, when defining the non-canonical renormalization process, we primarily use the following principles: (1) The Hamiltonian of the DNN feature layer output system has exactly the same form as the Hamiltonian of the input data system in Step 1. Moreover, in the partition function of this system, unlike a continuous field that can take arbitrary values over the whole function space, a neural network has its input-output functional relationship entirely determined by its parameters once the network architecture is fixed. The system therefore only possesses the energy states corresponding to the parameter values realized during the training process. Consequently, in the partition function of the feature layer output, the integral over the entire function space is replaced by an integral with respect to the image measure induced by the distribution of the network parameters during training. The finite-dimensional property of DNNs and their data avoid the issue in field theory of being unable to define a uniform measure on infinite-dimensional spaces, making our theory rigorous. (2) The main difference from the Ising model case is that, in addition to keeping the partition function invariant before and after the renormalization transformation, we use the renormalization group equations (RGEs) introduced by the RG method for continuous fields as a criterion. From the mathematical perspective, the one-dimensional Ising model has only a single characteristic parameter: the coupling constant, so the condition of partition function invariance is sufficient to reach the conclusion. In contrast, exponential family data with multiple characteristic parameters requires additional equations. A more fundamental difference is that the RGEs describe the change in the (connected) correlation functions of the system under the RG transformation. These equations have a remainder term, which vanishes at the fixed point where the scale parameter tends to infinity. The equations are exact only in the limiting case, and we are only concerned with this limit.

Finally in Step 3, when proving the equivalence between the non-canonical RG and the canonical RG, compared to the proof in \cite{ref13}, where the limit of the single characteristic parameter of the DNN feature layer output system being $0$ could be directly obtained using inequalities, for exponential family data with multiple characteristic parameters, it is necessary to first prove the existence of the limit of the characteristic parameters, and then combine the RGEs as conditions to derive that the limit of the characteristic parameters is the same as the fixed point under RG. In the proof, based on the definition of correlation functions, the multivariate Laplace transform over the complex domain and its inversion formula is introduced as a tool.

The result of this paper can be summarized as the following theorem:
\begin{theorem}
Let $F$ be the feature layer output of a FCDNN, parameters $\bm{W}^{(t)}$ trained by the simulated annealing algorithm with the learning rate chosen properly. The distribution density function of input data is of the following form that belongs to the exponential family:
\begin{align}
\mathbb{P}(x, \theta) = \frac{1}{\mathcal{Z(\theta)}} \mathrm{e}^{-\sum_{r=1}^{R} \theta_r \left( \sum_{i=1}^N x_i^2 \right)^r}, x\in \mathbb{R}^N,
\end{align}
where ${\theta}_r>0, r=1,2,\cdots,R$ are the characteristic parameters of the data. We assume: (1) The fixed points ${\theta}^{*}_r, r=1,2,\cdots,R$ exist under the canonical RG of continuous fields; (2) According to the RGEs, the connected correlation functions of the output system of DNN remain invariant through training. Then for the feature layer output system we have $\lim_{t \rightarrow \infty} \tilde{\theta}_r^t,\ r=1,2,\cdots,R$ exists and $\lim_{t \rightarrow \infty} \tilde{\theta}_r^t = {\theta}_r^*, r=1,2,\cdots,R$ under the condition $\tilde{\mathcal{Z}}(\tilde{\theta}_t)=\mathbb{E}[\mathrm{e}^{-\sum_{k} \tilde{\mathcal{H}}(F(x_k, \bm{W}^{(t)}),\tilde{\theta}_t)}]\equiv \mathcal{Z}(\theta)$.
\end{theorem} 

This theorem implies that when the simulated annealing training process of a FCDNN reaches optimality, for this class of exponential family data, the limits of the characteristic parameters of the feature layer output system are equal to the fixed point of the characteristic parameters of the data under the RG method. The network is equivalent to the RG method for continuous fields, and thus can effectively extract the macroscopic features in the physical sense of this continuously distributed data.

\subsection{Related works}

For the classification and enumeration of related works in the field of interpretability research, especially the introduction of research works also based on the RG method, we refer to Section 1.2 of \cite{ref13}. In this section, we introduce some new research works. 

In recent years, in the field of large language models (LLMs), there are some interpretability works. In 2023, OpenAI generated the natural language descriptions of the neurons in GPT-2 model by GPT-4 \cite{ref28}. In 2024, Anthropic used the sparse auto-encoder for dictionary learning to achieve the visualization of the features of Claude model \cite{ref29}. Benefiting from the reasoning capabilities of large models, Deepseek R1 \cite{ref30} employed a chain of Thought (CoT) approach to enhance model transparency and foster human trust. Building on this, OpenAI proposed specific methods for monitoring the CoT, identifying ambiguous expressions within it to further improve the interpretability of results \cite{ref31}.

Recently, the ``emergence" or ``aha moment" phenomena exhibited by large models, which differ from small models, have attracted attention. Research \cite{ref58} explained that reinforcement learning (RL) enables large models to resemble human reasoning processes: they first grasp underlying executions for problems and then explore higher-level planning. Under this hierarchical structure, the shift of the model's learning focus toward high-level strategies accounts for the occurrence of emergence phenomena. A recent study by FAIR \cite{ref59}, by establishing a mathematical framework, divided the training process of neural network models into three stages, demonstrating that the model initially engages in memorization, then learns independent features, and finally extracts interactive features. These studies confirm the existence of data features extracted by DNNs and provide an intuitive depiction of the feature extraction process, laying the foundation for our theory.

The first noteworthy work which introduced the RG method into the interpretability research is \cite{ref34}. In this paper the authors established a conceptual corresponding framework between the RBM and the variational RG. This was the first work that concretized the correspondence relationship based on the intuitive similarity between DNN and RG through mathematical expression. For a comparison between the correspondence framework we established and \cite{ref34} as well as its subsequent related studies, we refer to \cite{ref13}.

\subsection{Structure of the paper}
The rest of this paper is organized as follows. In Section 2, we introduce the relevant background knowledge for our work, including the DNN architecture used, the loss function, the training algorithm and its convergence lemma, as well as the RG method for continuous fields. In Section 3, we present the specific form of the data distribution employed, detail the methods for establishing both canonical and non-canonical renormalization processes on this distribution, and discuss the conclusions derived from these processes. In Section 4, we summarize the findings of this paper and outline potential directions for future research.

\section{Backgrounds}
\subsection{DNN: structure, data and training algorithm}
A DNN can be represented as a mapping $F$ with parameters $w$ from the input data space $\Omega$ to the target space $Y$. For a data point $x \in \Omega$, denote $\hat{y}(x,\bm{W})=F(x,\bm{W})$ as the output of the feature layer in the network. For the non-canonical renormalization process realized by the training of DNNs, we consider FCDNNs. The FCDNN is simple to analyze but can be seen as the prototype of other complicated network structures, therefore, it is representative. The sketch  of a FCDNN is shown in the following Figure 2:

\begin{figure}[H]
	\centering
	\includegraphics[width=0.4\textwidth]{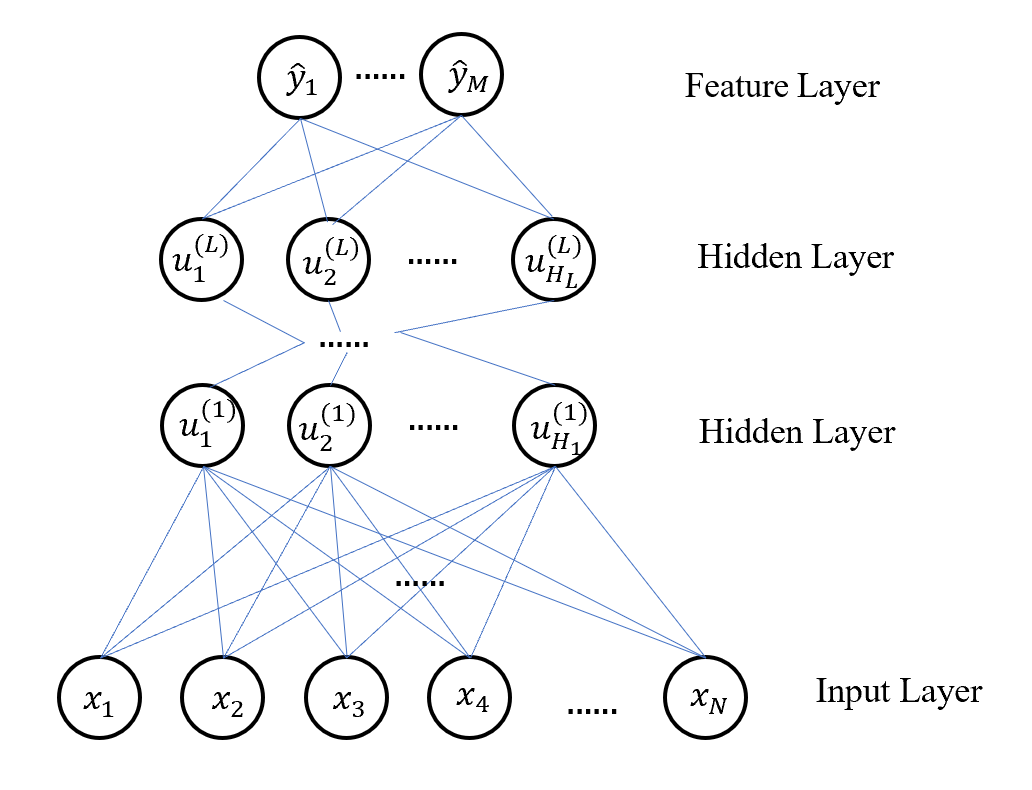}\\
	\caption{Fully Connected Network Structure}\label{structure}
\end{figure}

The input of this neural network is $\bm{x}=(x_1, x_2, \cdots, x_N) \in {\mathbb{R}}^N$, the feature layer output is $\bm{\hat{y}}=({\hat{y}}_{1}, {\hat{y}}_{2}, \cdots, {\hat{y}}_{M}) \in {\mathbb{R}}^M$. Between the input layer and the feature layer output result, there are $L$ hidden layers. The forward propagation process of the network can be referred to the description in \cite{ref13} Section 4. In practice, the commonly used expression of the loss function is:
\begin{align}
\widehat{L}(\bm{W}^{(t)})=\frac{1}{2S}\sum_{s=1}^S ||\hat{y}(\bm{W}^{(t)},{\bm{x}}^{(s)})-{\bm{y}}^{(s)}||_2^2.
\end{align}
Here $\bm{W}^{(t)}$ denotes the parameters at time $t$ of training in the network. The labelled training dataset is $\{({\bm{x}}^{(s)},{\bm{y}}^{(s)})\}_{s=1}^{S}$, where ${\bm{y}}^{(s)}$ are the target feature layer outputs. When $\bm{W}^*$ denotes one of the optimal values of the network parameters, we have $\hat{y}(\bm{W}^{*},{\bm{x}}^{(s)})={\bm{y}}^{(s)}, s=1,2,\cdots,S$ on the training data. Similar to \cite{ref13}, we adopt the following loss function with the regularization term (we omit the data input to the network since this has no impact on analysis):
\begin{equation}
L(\bm{W}^{(t)})=\frac{1}{2}||\hat{y}(\bm{W}^{(t)})-\hat{y}({\bm{W}}^*)||_2^2+{\tau}||\bm{W}^{(t)}-{\bm{W}}^*||^2_2.
\end{equation}
In this loss function, the second term is the regularization term, coefficient ${\tau}>0$ is a hyper-parameter. This term must be introduced since we consider the simulated annealing algorithm. For more details we refer to \cite{ref13} Remark 2.2.

The following is the form of the distribution density function of exponential family ($x_i \in \mathbb{R}, i = 1,2, \cdots, N, {\theta}_r \in \mathbb{R}, r = 1,2, \cdots, R$):
\begin{align}
\mathbb{P}(x_1,x_2,\cdot\cdot{\cdot},x_N;{\theta}_1,{\theta}_2,\cdot\cdot{\cdot},{\theta}_R)=C(\bm{\theta})\mathrm{e}^{\sum_{r=1}^{R}{\eta}_r(\bm{\theta})T_r(\bm{x})}h(\bm{x}).
\end{align}
Here, $ R \in \mathbb{N} $, $ \bm{\theta}{\triangleq}({\theta}_1,{\theta}_2,\cdot\cdot{\cdot},{\theta}_R) $ are the parameters in the distribution, $ C(\bm{\theta})>0,{\eta}_r(\bm{\theta}),r=1,\cdot\cdot{\cdot},R $ are functions of the parameters, $ h(\bm{x})>0,  T_r(\bm{x}) $ are functions of $ \bm{x} $. In statistics, to make $ T_r(\bm{x}) $ a sufficient and complete statistic for $ \bm{x} $, the following natural form of the exponential family distribution is usually adopted. It can be derived from the general form of the exponential family distribution mentioned above through parameter transformation. In this paper, we adopt the natural form of the exponential family distribution for analysis:
\begin{align}
\mathbb{P}(x_1,x_2,\cdot\cdot{\cdot},x_N;{\theta}_1,{\theta}_2,\cdot\cdot{\cdot},{\theta}_R)=C^*(\bm{\theta})\mathrm{e}^{\sum_{r=1}^{R}{\theta}_rT_r(\bm{x})}h(\bm{x}).
\end{align}

We consider using simulated annealing as the training algorithm of the network parameters. The continuation of the time-discrete simulated annealing algorithm is the following stochastic differential equation (according to \cite{ref49}):
\begin{equation} \label{eq:sde}
d\mathbf{W}^{(t)}=-{\nabla}L(\bm{W}^{(t)})dt+\sqrt{{\eta}_t}dB^{(t)},t{\geq}0.
\end{equation}
Here, $B^{(t)}$ is the standard Brownian motion with the same dimension as parameter $\bm{W}^{(t)}$, ${\eta}_t$ is a function of time $t$ called the learning rate. For the initial condition, we choose $\bm{W}^{(0)}$ follows a fixed distribution $P_0(\bm{W}^{(0)})$.

Holley and Stroock et al \cite{ref50, ref51} discussed the convergence in probability of this algorithm. In our analysis we use their conclusion, similar to theorem 2.2 in \cite{ref51}, we have:
\begin{lemma}\label{lemma1}
Suppose $\bm{W}^{(t)}$ satisfies equation (\ref{eq:sde}), then there exists a choice of the concrete form of the learning rate ${\eta_t}$, under which we have for any ${\delta}>0$:
\begin{align}
\mathbb{P}(L(\bm{W}^{(t)}){\geq}{\delta}){\leq}{\epsilon}({\delta},t),
\end{align}
where ${\epsilon}({\delta},t){\rightarrow}0$ when $t{\rightarrow}{\infty}$.
\end{lemma}
Similar to \cite{ref13}, this lemma is the main conclusion we used in our proof. Also, as in the case of the Ising model, our proof does not actually deal with the regularization term in the loss function and can be generalized to the more commonly used method in practice \textemdash stochastic gradient descent (SGD) algorithm (\cite{ref13} Remark 3.2).

\subsection{Renormalization group method: on continuous field}
\subsubsection{Renormalization group based on scaling transformation}

For a basic introduction to the renormalization group, refer to \cite{ref13} Section 2.2. The biggest difference from the discrete case of coarse-graining RG methods is that the RG for continuous fields utilizes scale transformation and RGEs. RGEs can take various forms based on different criteria. For instance, corresponding RGEs can be considered for the system's correlation functions, connected correlation functions, relative entropy, and among others. For the same system, the results obtained from RGEs established under different frameworks are consistent since they all keep the state functions of the system invariant to ensure the corresponding macro-scale system or thermodynamic system remains consistent under the RG transformation. Below, we present the basic framework of the RG method based on the RGEs satisfied by the system's multipoint (connected) correlation functions. Relevant references include \cite{ref54, ref55}.

First, in physics, normally in order to do the RG calculation on the system, there are some constraints on the system's Hamiltonian $ \mathcal{H}(\sigma) $ known as the Landau-Ginzburg-Wilson Hamiltonian (\cite{ref54} Section 9.1.2). The main purpose of this assumption is to endow the system with a kind of translational invariance.

In this paper we consider the RGEs on (connected) correlation functions. For a system described by a continuous field, its correlation function is defined by the partition function in an external variable field $H(x)$ (see \cite{ref54} Section 9.1.1 formula 2). We denote this partition function by $\hat{\mathcal{Z}}(H)$ here to avoid confusion with the original partition function $\mathcal{Z}$. It can be seen as the Laplace transform of the Boltzmann factor $ \mathrm{e}^{-\mathcal{H}} $ and performs as the generating function of correlation functions. Therefore, the correlation functions are defined via the functional derivative of $\hat{\mathcal{Z}}(H)$ and can be seen in the last formula of \cite{ref54} Section 9.1.1. For the connected correlation functions, their generating function is $\mathcal{W}(H)=\ln \hat{\mathcal{Z}}(H)$, and the definition is given by formula (9.5) of \cite{ref54}.

Then in the RG method for continuous fields, we introduce a scale parameter $\lambda$ and use it to construct the corresponding renormalization transformation. When $\lambda = 1$, it corresponds to the original model unchanged:
\begin{align}
\mathcal{H}_{\lambda=1}(\sigma){\triangleq}\mathcal{H}(\sigma).
\end{align}
When $\lambda > 1$, the renormalization transformation is denoted by:
\begin{align}
\mathcal{H}(\sigma){\mapsto}\mathcal{H}_{\lambda}(\sigma).
\end{align}
One specific method to construct a renormalization transformation is:
\begin{equation} \label{eq:rgtransform}
\mathcal{H}_{\lambda}(\sigma(x))=\mathcal{H}_{\lambda=1}(Z(\lambda)^{\frac{1}{2}}\sigma(\frac{x}{\lambda})).
\end{equation}

The RGE satisfied by the n-point connected correlation functions of the system is (see \cite{ref54} equation (9.9)):
\begin{equation}
\begin{split}
&{\mathcal{G}}_{\lambda}^{(n)}(\sigma(x_1),\sigma(x_2),{\cdot}{\cdot}{\cdot},\sigma(x_n))-Z^{-\frac{n}{2}}(\lambda){\mathcal{G}}^{(n)}(\sigma({\lambda}x_1),\sigma({\lambda}x_2),{\cdot}{\cdot}{\cdot},\sigma({\lambda}x_n))\\&={\mathcal{R}}_{\lambda}^{(n)}(\sigma(x_1),\sigma(x_2),{\cdot}{\cdot}{\cdot},\sigma(x_n)).
\end{split}
\end{equation}
In the above equations, the term $ {\mathcal{G}}_{\lambda}^{(n)} $ represents the n-point connected correlation function of the system after the renormalization transformation, i.e., when the Hamiltonian is $ \mathcal{H}_{\lambda}(\sigma) $. The Hamiltonian $ \mathcal{H}_{\lambda}(\sigma) $ is also referred to as the effective Hamiltonian at scale $\lambda$. $Z(\lambda)$ can be chosen arbitrarily as long as it ensures the validity of the RGE. The term $ {\mathcal{R}}_{{\lambda}}^{(n)} $ is the remainder of the equation, which satisfies the condition that it is a rapidly decreasing function of $\lambda$ as $ \lambda{\rightarrow}{\infty} $. Different forms of $Z(\lambda)$ and $ {\mathcal{R}}_{{\lambda}}^{(n)} $ correspond to different renormalization transformations. The RGE with the same form but $ {\mathcal{G}}_{\lambda}^{(n)} $ replaced by the n-point correlation function also holds.

In particular, in the case of the one-dimensional Ising model, the above RG framework is consistent with the real-space RG, and calculations based on the RGEs yield the same fixed point for the coupling constant. Refer to \cite{ref56} and \cite{ref57}.

\section{Method and results}

In this section we introduce our method and results concretely, following the three steps introduced in Section 1. First for the canonical RG part, we describe the input data of a FCDNN under the field theory framework and then give its RGEs. Then for the non-canonical RG done by training process of FCDNN, we describe it with the same framework using the RGEs with the same form. Finally we prove that the canonical RG and FCDNN training both derive the same limit of the parameters in the data distribution.
\subsection{Canonical renormalization}
\subsubsection{Data and transformation}
First we introduce the concrete form of the input data we consider for a FCDNN. Specifically, we consider the distribution density function of the exponential family introduced in Section 2.1:
\begin{align}
\mathbb{P}(\bm{x};\bm{\theta})=C(\bm{\theta})\mathrm{e}^{-\sum_{r=1}^{R}\theta_{r}T_{r}(\bm{x})}h(\bm{x}), \quad \bm{x}=(x_1,x_2,\cdots,x_N)\in\mathbb{R}^N.
\end{align}
Here, a negative sign is added in the exponent to align with the form of the Gibbs measure. We select
\begin{align}
T_{r}(\bm{x})=\left(\sum_{j,k=1}^{N}a_{jk}x_jx_k\right)^r,
\end{align}
that is, if the probability density of the system is interpreted in the form of a Gibbs distribution, the energy function of the system is a polynomial in $x_i, i=1,\cdots,N$ containing only even-degree terms, corresponding to a physical system with only even-order interactions. Here, $a_{jk}\in\mathbb{R},\ j,k=1,\cdots,N$ are constants such that the matrix $A=(a_{jk})$ is symmetric and positive definite. Additionally, we take:
\begin{align}
h(\bm{x})\equiv 1,
\end{align}
and $C(\bm{\theta})$ is the inverse number of the partition function:
\begin{align}
C(\bm{\theta})=\frac{1}{\mathcal{Z(\theta)}}, \quad \mathcal{Z(\theta)}=\int \mathrm{e}^{-\sum_{r=1}^{R}\theta_{r}T_{r}(\bm{x})}d\mathbf{x}=\int \mathrm{e}^{-\sum_{r=1}^{R}\theta_{r}\left(\sum_{j,k=1}^{N}a_{jk}x_jx_k\right)^r}d\mathbf{x}.
\end{align}
Thus, the probability density function of the considered data is:
\begin{align}
\mathbb{P}(\bm{x};\bm{\theta})=\frac{1}{\mathcal{Z(\theta)}}\mathrm{e}^{-\sum_{r=1}^{R}\theta_{r}\left(\sum_{j,k=1}^{N}a_{jk}x_jx_k\right)^r}.
\end{align}
Viewing this exponential family distribution as a physical model, the parameters $\theta_1,\theta_2,\cdots,\theta_R$ in the exponential family are analogous to the coupling constants in the one-dimensional Ising model, representing the strengths of the $2,4,\cdots,2R$-order interactions in the system, respectively. To ensure integrability, it is also required here that $\theta_1,\theta_2,\cdots,\theta_R>0$.

After determining the form of the model, according to the discussion of the canonical RG in Section 2, to apply the RG method, it is necessary to transform the Hamiltonian in the model into the form of a Landau-Ginzburg-Wilson Hamiltonian. In the present discussion, since a polynomial function is chosen, the regularity conditions required for the function are naturally satisfied; we only need to transform the model into a translationally invariant form. Here, we consider using a coordinate transformation. First, due to the symmetric positive definiteness of the matrix $A = (a_{jk})$, its eigenvalues $\mu_1, \mu_2, \cdots, \mu_N \geq 0$, and it can be orthogonally diagonalized:
\begin{align}
M = C^T A C,
\end{align}
\begin{align}
M=\left(
\begin{matrix}
     {\mu}_1 & 0 & 0 & 0 \\
     0 & {\mu}_2 & 0 & 0 \\
     0 & 0 & \cdot\cdot\cdot & 0 \\
     0 & 0 & 0 & {\mu}_N
\end{matrix}
\right).
\end{align}
Then, if we set:
\begin{align}
X = C Y, \quad X = (x_1, x_2, \cdots, x_N), \quad Y = (y_1, y_2, \cdots, y_N),
\end{align}
we have:
\begin{align}
X^T A X = Y^T C^T A C Y = Y^T M Y,
\end{align}
thus transforming the quadratic form into its standard form:
\begin{align}
\sum_{j,k=1}^{N} a_{jk} x_j x_k = \sum_{i=1}^N \mu_i y_i^2.
\end{align}
Furthermore, to ensure the aforementioned translational invariance, we can perform an additional scaling transformation. Define the matrix (where if any $\mu_i = 0$, the corresponding diagonal element is set to $0$):
\begin{align}
{M}^{'}=\left(
\begin{matrix}
     \frac{1}{\sqrt{{\mu}_1}} & 0 & 0 & 0 \\
     0 & \frac{1}{\sqrt{{\mu}_2}} & 0 & 0 \\
     0 & 0 & \cdot\cdot\cdot & 0 \\
     0 & 0 & 0 & \frac{1}{\sqrt{{\mu}_N}}
\end{matrix}
\right),
\end{align}
and take:
\begin{align}
X = C M' Y,
\end{align}
then the Hamiltonian can be transformed into:
\begin{align}
\mathcal{H} = \sum_{r=1}^{R} \theta_r \left( \sum_{j,k=1}^{N} a_{jk} x_j x_k \right)^r = \sum_{r=1}^{R} \theta_r \left( \sum_{i=1}^N y_i^2 \right)^r.
\end{align}
This is exactly the form of a Landau-Ginzburg-Wilson Hamiltonian. In the following, we will analyse based on this form. For consistency in the subsequent discussion, the transformed variables are still denoted by $x_i$, and the probability density function of the data distribution is written as:
\begin{align}
\mathbb{P}(x)=C({\theta})\mathrm{e}^{-{\sum}_{r=1}^{R}{\theta}_r({\sum}_{i=1}^Nx_i^2)^r},
\end{align}
where $ C({\theta}) $ is a constant equals to the inverse number of partition function. In the following, it will remain invariant and we only perform renormalization on the exponential part.

\subsubsection{Description of the canonical renormalization on exponential family}

In this section we realize Step 1. We regard the input data set as a statistical physical system, this means that we write the distribution density function of the data as the Gibbs form:
\begin{align}
\mathbb{P}(x, \theta)= \frac{1}{\mathcal{Z}(\theta)}\mathrm{e}^{-\mathcal{H}(x,\theta)},
\end{align}
where $\theta$ denotes the characteristic parameter, $\mathcal{H}(x,\theta)$ is the Hamiltonian, $\mathcal{Z}(\theta)$ is the partition function. After transformation in last section, the data probability density function is:
\begin{align}
\mathbb{P}(x) = \frac{1}{\mathcal{Z(\theta)}} \mathrm{e}^{-\sum_{r=1}^{R} \theta_r \left( \sum_{i=1}^N x_i^2 \right)^r}.
\end{align}
According to Section 1.2, the data distribution is on ${\mathbb{R}}^N$ and in the framework of continuous field, the field corresponding with this data is then defined as:
\begin{equation}
\begin{split}
{\sigma}:&\mathbb{R}^N {\rightarrow} \mathbb{R}^N\\
&{\sigma}_i(x)=x_i, i=1, 2, \cdot\cdot\cdot, N.
\end{split}
\end{equation}
The Hamiltonian of the system is:
\begin{align}
\mathcal{H}(x, \theta)= \sum_{r=1}^{R} \theta_r \left( \sum_{i=1}^N x_i^2 \right)^r.
\end{align}
Based on the above representation, we can derive the renormalization transformation and the corresponding RGE for data following this class of exponential family distributions, as discussed in Section 2. First, the renormalization transformation is given by (\ref{eq:rgtransform}). Substituting the specific form of the field $\sigma(x)$ chosen above, the renormalization transformation for the data distribution can be obtained as:
\begin{equation} \label{eq:rgtrans}
x_i \mapsto Z^{\frac{1}{2}}(\lambda) \frac{x_i}{\lambda}, \quad i = 1, 2, \cdots, N.
\end{equation}
Given the specific form of the Hamiltonian here, the renormalization transformation coefficients for $x$ can be absorbed into the parameters $\theta$:
\begin{equation}
\begin{split}
\mathcal{H}_{\lambda}(\sigma(x))&={\sum}_{r=1}^R{\theta}_r({\sum}_{i=1}^N(Z(\lambda)^{\frac{1}{2}}\frac{x_i}{\lambda})^2)^r\\&={\sum}_{r=1}^R({\theta}_rZ(\lambda)^r\frac{1}{{\lambda}^{2r}})({\sum}_{i=1}^N x_i^2)^r,
\end{split}
\end{equation}
thus, the renormalization transformation for the parameters can be directly written as:
\begin{align}
\theta_r \mapsto \theta_r \cdot Z(\lambda)^r \cdot \frac{1}{\lambda^{2r}} \triangleq \theta_r^{\lambda}.
\end{align}
When the scale parameter $\lambda \rightarrow \infty$, the corresponding parameters represent the fixed points of the renormalization transformation:
\begin{align}
\theta_r^* \triangleq \theta_r \cdot \lim_{\lambda \rightarrow \infty} \left( Z(\lambda)^r \cdot \frac{1}{\lambda^{2r}} \right).
\end{align}
Here, based on the specific form of the field $\sigma(x)$, the $n$-point correlation function should be evaluated at $n$ points $x^{(1)}, x^{(2)}, \cdots, x^{(n)}$ in the field's definition space $\mathbb{R}^N$. For notational convenience, we use multi-indices:
\begin{align}
x^{(i)} = (x^{(i)}_1, x^{(i)}_2, \cdots, x^{(i)}_N), \quad \vec{\kappa} = (\kappa_1, \kappa_2, \cdots, \kappa_N), \quad (x^{(i)})^{\vec{\kappa}} \triangleq (x^{(i)}_1)^{\kappa_1} (x^{(i)}_2)^{\kappa_2} \cdots (x^{(i)}_N)^{\kappa_N}.
\end{align}
Then, the $n$-point correlation function is:
\begin{equation}
\begin{split}
&{\mathcal{G}}^{(n)}(\sigma(x^{(1)}),\sigma(x^{(2)}),{\cdot}{\cdot}{\cdot},\sigma(x^{(n)});{\vec{\kappa}}_1,{\vec{\kappa}}_2,{\cdot}{\cdot}{\cdot},{\vec{\kappa}}_n)\\&=\frac{1}{\mathcal{Z}}{\int}(x^{(1)})^{\vec{\kappa}_1}(x^{(2)})^{\vec{\kappa}_2}{\cdot}{\cdot}{\cdot}(x^{(n)})^{\vec{\kappa}_n}\mathrm{e}^{-{\sum}_{r=1}^R{\theta}_r({\sum}_{i=1}^Nx_i^2)^r}dx_1dx_2{\cdot}{\cdot}{\cdot}dx_N.
\end{split}
\end{equation}
Similarly, we can write the correlation functions $\mathcal{G}_{\lambda}^{(n)}(\sigma(x^{(1)}), \sigma(x^{(2)}), \cdots, \sigma(x^{(n)}))$ and \\ $\mathcal{G}^{(n)}(\sigma(\lambda x^{(1)}), \sigma(\lambda x^{(2)}), \cdots, \sigma(\lambda x^{(n)}))$. We then obtain the specific form of the RGE as:
\begin{equation}
\begin{split}
&\frac{1}{\mathcal{Z}_{\lambda}}{\int}(x^{(1)})^{\vec{\kappa}_1}(x^{(2)})^{\vec{\kappa}_2}{\cdot}{\cdot}{\cdot}(x^{(n)})^{\vec{\kappa}_n}\mathrm{e}^{-{\sum}_{r=1}^R{\theta}_r(Z(\lambda)\frac{1}{{\lambda}^2})^r({\sum}_{i=1}^Nx_i^2)^r}dx_1dx_2{\cdot}{\cdot}{\cdot}dx_N\\&-Z^{-\frac{n}{2}}(\lambda)\frac{1}{\mathcal{Z}}{\int}{\lambda}^{{\sum}_{l=1}^n|\vec{\kappa}_l|}{\cdot}(x^{(1)})^{\vec{\kappa}_1}(x^{(2)})^{\vec{\kappa}_2}{\cdot}{\cdot}{\cdot}(x^{(n)})^{\vec{\kappa}_n}\mathrm{e}^{-{\sum}_{r=1}^R{\theta}_r({\sum}_{i=1}^Nx_i^2)^r}dx_1dx_2{\cdot}{\cdot}{\cdot}dx_N\\&={\mathcal{R}}_{\lambda}^{(n)}(x^{(1)},x^{(2)},{\cdot}{\cdot}{\cdot},x^{(n)}).
\end{split}
\end{equation}
Here, $\mathcal{R}_{\lambda}^{(n)}(x^{(1)}, x^{(2)}, \cdots, x^{(n)})$ is a rapidly decreasing function with respect to the scale transformation parameter $\lambda$.

When the scale parameter $\lambda \rightarrow \infty$, combined with the condition that the partition function remains invariant $\mathcal{Z}_{\lambda} = \mathcal{Z}$, we obtain the RGE satisfied by the fixed point:
\begin{equation}
\begin{split}
&{\int}(x^{(1)})^{\vec{\kappa}_1}(x^{(2)})^{\vec{\kappa}_2}{\cdot}{\cdot}{\cdot}(x^{(n)})^{\vec{\kappa}_n}\mathrm{e}^{-{\sum}_{r=1}^R{\theta}_r^*({\sum}_{i=1}^Nx_i^2)^r}dx_1dx_2{\cdot}{\cdot}{\cdot}dx_N\\&-{\lim}_{{\lambda}{\rightarrow}{\infty}}(Z^{-\frac{n}{2}}(\lambda){\lambda}^{{\sum}_{l=1}^n|\vec{\kappa}_l|}){\int}{\cdot}(x^{(1)})^{\vec{\kappa}_1}(x^{(2)})^{\vec{\kappa}_2}{\cdot}{\cdot}{\cdot}(x^{(n)})^{\vec{\kappa}_n}\mathrm{e}^{-{\sum}_{r=1}^R{\theta}_r({\sum}_{i=1}^Nx_i^2)^r}dx_1dx_2{\cdot}{\cdot}{\cdot}dx_N=0.
\end{split}
\end{equation}
Since the RGE cannot be solved explicitly, we need to analyze it to draw the necessary conclusions. Therefore, in the following discussion, to ensure the meaningfulness of our analysis, we assume that the above RG method can be applied to this specific data model we have selected and that non-trivial parameter fixed points exist.

\subsection{Non-canonical renormalization and the correspondence relationship}

\subsubsection{Description of the non-canonical renormalization on exponential family}

In this section we realize the Step 2. Similar to the previous canonical situation, to interpret the DNN's processing of data as a renormalization transformation, we first need to describe the feature layer output of a DNN using the language of field theory. In this context, the continuous field can naturally be defined as the feature layer output of the network via the forward propagation computation:
\begin{equation}
\begin{split}
{\phi}:&\mathbb{R}^N {\rightarrow} \mathbb{R}^M\\
&{\phi}_j(x)=F_j(x;\bm{W}), j=1, 2, \cdot\cdot\cdot, M.
\end{split}
\end{equation}
Here, $F_j(x; \bm{W})$ represents the expression of the $j$-th output of the feature layer, and $\bm{W}$ denotes the parameters. We still denote the training data set of the network as $X=\{x_k\}_{k=1}^S$. Instead of integrating over the whole $\mathbb{R}^N$ in the unified form of Hamiltonian, the energy of the feature layer output system is just concentrated on the sample points in the training process. The Hamiltonian for the neural network's feature layer output is defined as:
\begin{equation}
\begin{split}
\tilde{\mathcal{H}}({\phi},{\theta},M)=&{\sum}_{k=1}^S {\sum}_{r=1}^R{\tilde{\theta}}_r^t({\sum}_{j=1}^M{\phi}_j(x_k,\bm{W}^{(t)})^2)^r\\&={\sum}_{k=1}^S {\sum}_{r=1}^R{\tilde{\theta}}_r^t({\sum}_{j=1}^MF_j(x_k;\bm{W}^{(t)})^2)^r.
\end{split}
\end{equation}
This definition originates from the principle of the RG method: the system before and after the renormalization transformation must have Hamiltonians of exactly the same form. Therefore, when the neural network training process is regarded as a non-canonical renormalization process, the feature layer output Hamiltonian has exactly the same expression as $\mathcal{H}(x, \theta)$ with $x$ replaced by $F(x, \bm{W}^{(t)})$ and $\theta$ replaced by $\tilde{\theta}_t$.

For greater clarity, we rewrite $\tilde{\mathcal{H}}({\phi},{\theta},M)=\tilde{\mathcal{H}}(\phi, \tilde{\theta}^t, \bm{W}^{(t)}, M) $ here. The characteristic parameters $\tilde{\theta}_r^t$ are functions of training time $t$ satisfying $\tilde{\theta}_r^t>0$ and $\tilde{\theta}_r^0 ={\theta}_r, \forall r$. The corresponding partition function is (note that ${\phi}_j(x)=F_j(x;\bm{W})$):
\begin{equation}  \label{eq:ncHamilton}
\begin{split}
\mathcal{Z}_t&=\mathbb{E}_{\bm{W}^{(t)}}\big[\mathrm{e}^{Ng(t)-\tilde{\mathcal{H}}(\phi, \tilde{\theta}^t, \bm{W}^{(t)}, M)}\big]\\&=\mathbb{E}_{\bm{W}^{(t)}}\big[\mathrm{e}^{Ng(t)-{\sum}_{k=1}^S {\sum}_{r=1}^R{\tilde{\theta}}_r^t({\sum}_{j=1}^M{\phi}_j(x_k;\bm{W}^{(t)})^2)^r}\big].
\end{split}
\end{equation}
$g(t)$ is a non-stochastic function of $t$ with no singularity. It corresponds to the non-singular part of the system's Hamiltonian. The necessity of this term and detailed discussion can be found in \cite{ref13}. As discussed earlier, the field ${\phi}_j(x)=F_j(x;\bm{W})$ we consider here is completely determined by $\bm{W}$, which belongs to the finite-dimensional parameter space $\mathbb{R}^{N_p}$, where $N_p$ is the number of parameters in the DNN. Therefore, the non-strict integral ${\int}_{\Sigma} \cdot [d\sigma]$ in \cite{ref54} is replaced by ${\int}_{\mathbb{R}^{N_p}} {\cdot}\ d\phi$ in the case of non-canonical RG. Moreover, when we consider the simulated annealing training process, the integral is with respect to the image measure of $\bm{W}^{(t)}$, which is represented by $\mathbb{E}_{\bm{W}^{(t)}}[\cdot]$ in (\ref{eq:ncHamilton}).
From the perspective of a FCDNN as a continuous field defined on the space $\mathbb{R}^N$, the Hamiltonian expression above indicates that the network performs a non-canonical renormalization transformation on the input data. Time $t$ corresponds to the scale transformation parameter $\lambda$. 

Like RGE in the canonical RG, in the non-canonical RG there are criteria for the (connected) correlation functions. The theoretical basis we consider is that, in the theory of thermodynamics and statistical physics, if two statistical physical systems share identical numerical values, functional forms, and coupling constants in their partition functions, then their (connected) correlation functions must necessarily be identical or satisfy the same equations. Therefore, in the non-canonical RG, the connected correlation functions should also remain invariant. To be consistent with the RG method in \cite{ref54}, we consider the connected correlation functions here. The results are equivalent to the case of correlation functions. First, we give the definition of connected correlation functions of the DNN's feature layer output system. Similar to the definition in \cite{ref54}, we write
\begin{equation}
\hat{\mathcal{Z}}_*(\{{\mu}_{jk}\},\{{\tilde{\theta}}_r^*\}) \triangleq \int_{\mathbb{R}^{MS}_+} \mathrm{e}^{Ng_*-{\sum}_{k=1}^S {\sum}_{r=1}^R{\tilde{\theta}}_r^*({\sum}_{j=1}^M \phi_{jk}^2)^r-\sum_{j=1}^M \sum_{k=1}^S \mu_{jk}{\phi}_{jk}} \Pi_{jk}d\phi_{jk},
\end{equation}
where $g_* \triangleq {\lim}_{t{\rightarrow}{\infty}}g(t)$, $\mu_{jk} \in \mathbb{C}$, $\phi_{jk}={\phi}_j(x_k;\bm{W}^{*})$; we omit $\bm{W}^{*}$ here and express it in the subscript in $\hat{\mathcal{Z}}_*$, and the existence of $\tilde{\theta}_r^* \triangleq \lim_{t \rightarrow \infty} \tilde{\theta}_r^t$ will be discussed in the proof of Theorem 3.1. Here we consider the case of $t \rightarrow \infty$ since we only care about the limit of the characteristic parameters and also under this limit the remainder term of RGEs in both canonical and non-canonical RG vanish. $ \hat{\mathcal{Z}}_*(\{{\mu}_{jk}\},\{{\tilde{\theta}}_r^t\}) $ is the Laplace transform of
\begin{equation} 
f_*(\{{\phi}_{jk}\},\{{\tilde{\theta}}_r^*\})=\mathrm{e}^{Ng_*-{\sum}_{k=1}^S {\sum}_{r=1}^R{\tilde{\theta}}_r^*({\sum}_{j=1}^M{\phi}_{jk}^2)^r}. 
\end{equation}
Then with the multi-indices $ \vec{\kappa}^{(k)}=(\kappa_1^{(1)},\kappa_2^{(1)},\cdots, \kappa_M^{(1)})$ on each of the data points and vectors $\vec{\phi}_k=(\phi_{1k}, \cdots, \phi_{Mk}),\ k=1,2,\cdots,S $, the correlation functions are:
\begin{equation}
\begin{split}
&\tilde{\mathcal{G}}^*(x_1, \cdots, x_S; \vec{\kappa}^{(1)},\cdots,\vec{\kappa}^{(S)}) \triangleq \frac{\partial \hat{\mathcal{Z}}_*(\{{\mu}_{jk}\},\{{\tilde{\theta}}_r^*\}) }{{\partial}^{\kappa_1^{(1)}} \mu_{11} \cdots {\partial}^{\kappa_M^{(1)}} \mu_{M1} \cdots {\partial}^{\kappa_1^{(S)}} \mu_{1S} \cdots {\partial}^{\kappa_M^{(S)}}\mu_{MS}}|_{\mu_{jk}=0}\\&=\int_{\mathbb{R}^{MS}_+} {\vec{\phi}_1}^{\ \vec{\kappa}^{(1)}} \cdots {\vec{\phi}_S}^{\ \vec{\kappa}^{(S)}} \mathrm{e}^{Ng_*-{\sum}_{k=1}^S {\sum}_{r=1}^R{\tilde{\theta}}_r^*({\sum}_{j=1}^M{\phi}_{jk}^2)^r} \Pi_{jk}d\phi_{jk}.
\end{split}
\end{equation}
The connected correlation functions are:
\begin{equation}  \label{eq:noncanocorre}
\tilde{\mathcal{G}}_c^*(x_1, \cdots, x_S; \vec{\kappa}^{(1)},\cdots,\vec{\kappa}^{(S)}) \triangleq \frac{\partial \ln \hat{\mathcal{Z}}_*(\{{\mu}_{jk}\},\{{\tilde{\theta}}_r^*\}) }{{\partial}^{\kappa_1^{(1)}} \mu_{11} \cdots {\partial}^{\kappa_M^{(1)}} \mu_{M1} \cdots {\partial}^{\kappa_1^{(S)}} \mu_{1S} \cdots {\partial}^{\kappa_M^{(S)}}\mu_{MS}}|_{\mu_{jk}=0}.
\end{equation}
More details of the well-posedness of these connected correlation functions will be discussed in the proof of Theorem 3.1. Here we demonstrate our assumption: the invariance of the connected correlation function through the whole training process of the DNN. This assumption arises from the fact that in the RGEs, when the parameter $\lambda$ takes the initial value $1$ or tends to infinity, the remainder term vanishes, and the corresponding correlation functions remain unchanged. Therefore, in non-canonical renormalization, as $t \rightarrow \infty$, the correlation functions should also be equal to those at the initial state $t = 0$. The expression is the following equality:
\begin{equation} \label{eq:connectassump}
\tilde{\mathcal{G}}_c^*(x_1, \cdots, x_S; \vec{\kappa}^{(1)},\cdots,\vec{\kappa}^{(S)}) \equiv \tilde{\mathcal{G}}_c^0(x_1, \cdots, x_S; \vec{\kappa}^{(1)},\cdots,\vec{\kappa}^{(S)}),
\end{equation}
where $\tilde{\mathcal{G}}_c^0(x_1, \cdots, x_S; \vec{\kappa}^{(1)},\cdots,\vec{\kappa}^{(S)})=\frac{\partial \ln \hat{\mathcal{Z}}_0(\{{\mu}_{jk}\},\{{\tilde{\theta}}_r^0\}) }{{\partial}^{\kappa_1^{(1)}} \mu_{11} \cdots {\partial}^{\kappa_M^{(1)}} \mu_{M1} \cdots {\partial}^{\kappa_1^{(S)}} \mu_{1S} \cdots {\partial}^{\kappa_M^{(S)}}\mu_{MS}}|_{\mu_{jk}=0}$ and 
\begin{equation} 
\hat{\mathcal{Z}}_0(\{{\mu}_{jk}\},\{{\tilde{\theta}}_r^0\})=\int_{\mathbb{R}^{MS}_+} \mathrm{e}^{Ng(0)-{\sum}_{k=1}^S {\sum}_{r=1}^R{{\theta}}_r({\sum}_{j=1}^M \phi_{jk}^2)^r-\sum_{j=1}^M \sum_{k=1}^S \mu_{jk}{\phi}_{jk}} \Pi_{jk}d\phi_{jk},
\end{equation}
where $\phi_{jk}={\phi}_j(x_k;\bm{W}^{(0)})$ here. This is the criterion we keep in the non-canonical RG process.

\subsubsection{Method to build the correspondence relationship and results}

Based on the discussions in the previous two subsections, we have formulated both the original input data and the feature layer output of the DNN's forward propagation process using the language of field theory as two distinct physical systems defined on the space $\mathbb{R}^N$. These systems possess mutually corresponding interaction strength parameters $\theta_r$ and $\tilde{\theta}_r^t$. Moreover, we have provided a theoretical description of their respective renormalization processes. Therefore, in Step 3, the RG method and the DNN training process can now be directly compared. The method of comparison is the same as in the paper \cite{ref14}: we can demonstrate that the fixed points of the corresponding parameters under two processes for both systems are equal, thereby showing that the macroscopic features extracted by the DNN are identical to those obtained by the RG method. The following theorem presents our conclusion:
\begin{theorem}
When the distribution density function of the input data to the neural network is:
\begin{equation} \label{eq:expfam}
\mathbb{P}(x) = \frac{1}{\mathcal{Z(\theta)}} \mathrm{e}^{-\sum_{r=1}^{R} \theta_r \left( \sum_{i=1}^N x_i^2 \right)^r}, \quad x \in \mathbb{R}^N.
\end{equation}
Here, $\theta_r > 0, r=1,2,\cdots,R$ are the parameters of the data. Assume that the renormalization group method, using the scale transformation-based renormalization transformation \eqref{eq:rgtrans}, yields fixed points $\theta_r^* \geq 0, r=1,2,\cdots,R$ for the system parameters. The neural network architecture is the FCDNN, with the loss function defined as:
\begin{equation}
\begin{split}
L(\bm{W})=&\frac{1}{2S}\sum_{x_k \in X}[\sum_{j=1}^{M}(f(u^{(L)}_1w^{(L+1)}_{1j}+u^{(L)}_2w^{(L+1)}_{2j}+\cdot\cdot\cdot+u^{(L)}_{H_{L}}w^{(L+1)}_{H_{L}j}+b^{(L+1)}_{j})\\&-f(u^{*(L)}_1w^{*(L+1)}_{1j}+u^{*(L)}_2w^{*(L+1)}_{2j}+\cdot\cdot\cdot+u^{*(L)}_{H_{L}}w^{*(L+1)}_{H_{L}j}+b^{*(L+1)}_{j}))^2]\\&+{\tau}(\sum_{l=1}^{L+1}\sum_{i=1}^{H_{l-1}}\sum_{j=1}^{H_{l}}(w_{ij}^{(l)}-w_{ij}^{*(l)})^2+\sum_{l=1}^{L+1}\sum_{i=1}^{H_{l}}(b^{(l)}_{i}-b^{*(l)}_{i})^2).
\end{split}
\end{equation}
In $L(\bm{W})$, $X=\{x_k\}_{k=1}^S$ is the training data, $w_{ij}^{(l)}, b^{(l)}_{i}$ are the weight and bias parameters, $u$ denotes the intermediate results before feature layer and $f$ is the activation function (details can be seen in \cite{ref13} Section 4). For this loss function, the unique global optimal point of the parameters is $\bm{W}^{*}$. Using simulated annealing for training, the learning rate $\eta(t)$ is selected according to Lemma 2.1. Then, there exists a function $g(t) \in C^{\infty}[0, +\infty)$ such that, under the condition of equal partition functions:
\begin{align}
\mathcal{Z}_t = \mathcal{Z},
\end{align}
and the invariance of connected correlation functions (\ref{eq:connectassump}), we have:
for $r=1,2,\cdots,R$, the limit $\tilde{\theta}_r^* \triangleq \lim_{t \rightarrow \infty} \tilde{\theta}_r^t$ exists, and $\tilde{\theta}_r^*=\theta_r^*.$
\end{theorem}

{\bf Proof}

Recall that $\tilde{\mathcal{H}}(\phi, \tilde{\theta}^t, \bm{W}^{(t)}, M)={\sum}_{k=1}^S {\sum}_{r=1}^R{\tilde{\theta}}_r^t({\sum}_{j=1}^MF_j(x_k;\bm{W}^{(t)})^2)^r$. The condition of remaining the partition function invariant is:
\begin{equation}
\mathcal{Z}={\int}\mathrm{e}^{-{\sum}_{r=1}^{R}{\theta}_r({\sum}_{i=1}^Nx_i^2)^r}dx \equiv \mathcal{Z}_t=\mathbb{E}_{\bm{W}^{(t)}}\big[\mathrm{e}^{Ng(t)-\tilde{\mathcal{H}}(\phi, \tilde{\theta}^t, \bm{W}^{(t)}, M)}\big],
\end{equation}
we take the function
\begin{align}
g(t)=(1-\mathrm{e}^{-(t-1)})\frac{1}{N}\ln (\frac{\mathcal{Z}}{\mathrm{e}^{-\tilde{\mathcal{H}}(\phi, {\theta}^*, \bm{W}^{*}, M)}}),
\end{align}
which satisfies $g(0)=0$, and
\begin{align}
\frac{\mathcal{Z}}{\mathrm{e}^{Ng_*}}=\frac{\mathcal{Z}}{\mathrm{e}^{N{\lim}_{t{\rightarrow}{\infty}}g(t)}}=\mathrm{e}^{-\tilde{\mathcal{H}}(\phi, {\theta}^*, \bm{W}^{*}, M)}.
\end{align}
We rewrite the loss function as:
\begin{equation}
\begin{split}
L(\bm{W}^{(t)})&=\frac{1}{2S} \sum_{k=1}^S \sum_{j=1}^M (F_j(x_k,\bm{W}^{(t)})-F_j(x_k,\bm{W}^{*}))^2\\&+\tau (\sum_{l=1}^{L+1}\sum_{i=1}^{H_{l-1}}\sum_{j=1}^{H_{l}}(w_{ij}^{(l)}-w_{ij}^{*(l)})^2+\sum_{l=1}^{L+1}\sum_{i=1}^{H_{l}}(b^{(l)}_{i}-b^{*(l)}_{i})^2).
\end{split}
\end{equation}
For $j=1,2,\cdots,M,\ k=1,2,\cdots,S$, let
\begin{equation}
\hat{y}_{kj}\triangleq F_j(x_k,\bm{W}^{(t)}),\ y_{kj}\triangleq F_j(x_k,\bm{W}^{*}),
\end{equation}
then let
\begin{equation}
A_{kt}\triangleq \sum_{j=1}^M \hat{y}_{kj}^2,\ B_k\triangleq \sum_{j=1}^M y_{kj}^2,
\end{equation}
for $r=1,2,\cdots,R$, let
\begin{equation}
H_r(\bm{W}^{(t)})\triangleq \sum_{k=1}^S A_{kt}^r,\ H_r(\bm{W}^{*})\triangleq \sum_{k=1}^S B_k^r,
\end{equation}
and for any fixed constant ${\delta}_1>0$,
\begin{equation}
H_{At} \triangleq  \{\forall r, \ |H_r(\bm{W}^{(t)})-H_r(\bm{W}^{*})| \leq {\delta}_1\},\ 
H_{At}^c=\{\exists r, \ |H_r(\bm{W}^{(t)})-H_r(\bm{W}^{*})| >{\delta}_1\}.
\end{equation}
We can rewrite
\begin{equation} 
H_{At}^c=\{\exists r, \ |\sum_{k=1}^S (A_{kt}^r-B_k^r)| >{\delta}_1\}.
\end{equation}
Since $\hat{y}_{kj}, y_{kj} \in (0,1)$, we have $A_{kt}, B_k \in (0,M).$ Then $|A_{kt}^r-B_k^r| \leq rM^{r-1}|A_{kt}-B_k|$. So when $r$ satisfies $|\sum_{k=1}^S (A_{kt}^r-B_k^r)| >{\delta}_1$, we have 
\begin{equation}  \label{eq:neq2}
\sum_{k=1}^S |A_{kt}-B_k| \geq \frac{{\delta}_1}{rM^{r-1}}.
\end{equation}
For $a,b \in (0,+\infty)$, one can prove that
\begin{equation} \label{eq:neq1}
(\sqrt{a}-\sqrt{b})^2 \geq \frac{(a-b)^2}{4\max(a,b)},
\end{equation}
We consider for $\forall k$,
\begin{equation} 
\sum_{j=1}^M (\hat{y}_{kj}-y_{kj})^2=\sum_{j=1}^M \hat{y}_{kj}^2 + \sum_{j=1}^M y_{kj}^2-2\sum_{j=1}^M \hat{y}_{kj}y_{kj} \geq A_{kt}+B_k-2\sqrt{A_{kt}B_k}=(\sqrt{A_{kt}}-\sqrt{B_k})^2.
\end{equation}
Combine with (\ref{eq:neq1}) we have
\begin{equation} 
\sum_{j=1}^M (\hat{y}_{kj}-y_{kj})^2 \geq \frac{(A_{kt}-B_k)^2}{4M}.
\end{equation}
Notice that
\begin{equation} 
(\sum_{k=1}^S |A_{kt}-B_k|)^2 \leq S(\sum_{k=1}^S|A_{kt}-B_k|^2),
\end{equation}
use (\ref{eq:neq2}),
\begin{equation} 
\sum_{k=1}^S |A_{kt}-B_k|^2 \geq \frac{1}{S} (\frac{{\delta}_1}{rM^{r-1}})^2.
\end{equation}
Back to the loss function,
\begin{equation} 
L(\bm{W}^{(t)}) \geq \frac{1}{2S} \sum_{k=1}^S \sum_{j=1}^M (\hat{y}_{kj}-y_{kj})^2 \geq \frac{1}{2S} \sum_{k=1}^S \frac{(A_{kt}-B_k)^2}{4M} \geq \frac{1}{2S} \frac{1}{4MS}(\frac{{\delta}_1}{rM^{r-1}})^2
\end{equation}
Let ${\delta}_2 \triangleq \frac{{\delta}_1^2}{8S^2r^2M^{2r-1}}>0$, we have:
\begin{equation} 
H_{At}^c \subset \{\exists {\delta}_2>0,\ L(\bm{W}^{(t)})\geq {\delta}_2\}.
\end{equation}
Let
\begin{equation} 
{\Lambda}_t \triangleq \frac{\mathcal{Z}_t}{\mathrm{e}^{Ng(t)}}=\frac{\mathcal{Z}}{\mathrm{e}^{Ng(t)}},
\end{equation}
Then
\begin{equation} 
\lim_{t \rightarrow \infty} {\Lambda}_t \triangleq {\Lambda}_*=\mathrm{e}^{-{\sum}_{k=1}^S {\sum}_{r=1}^R {\theta}_r^*({\sum}_{j=1}^MF_j(x_k;\bm{W}^{*})^2)^r}.
\end{equation}
On the other hand,
\begin{equation} 
{\Lambda}_t=\mathbb{E}_{\bm{W}^{(t)}}[\mathrm{e}^{-{\sum}_{k=1}^S {\sum}_{r=1}^R {\tilde{\theta}}_r^t({\sum}_{j=1}^MF_j(x_k;\bm{W}^{(t)})^2)^r}].
\end{equation}
We perform the following decomposition and denote:
\begin{equation}
\begin{split}
{\Lambda}_t&=\mathbb{E}_{\bm{W}^{(t)}}[\mathrm{e}^{-{\sum}_{r=1}^R {\tilde{\theta}}_r^t H_r(\bm{W}^{(t)})}\mathbbm{1}_{H_{At}}]+\mathbb{E}_{\bm{W}^{(t)}}[\mathrm{e}^{-{\sum}_{r=1}^R {\tilde{\theta}}_r^t H_r(\bm{W}^{(t)})}\mathbbm{1}_{H_{At}^c}] \\&\triangleq I+II.
\end{split}
\end{equation}
On $H_{At}$, we have $\forall r,\ H_r(\bm{W}^{*})-{\delta}_1 \leq H_r(\bm{W}^{(t)}) \leq H_r(\bm{W}^{*})+{\delta}_1$, since ${\tilde{\theta}}_r^t>0$,
\begin{equation} \label{eq:neqI}
\mathrm{e}^{-{\sum}_{r=1}^R {\tilde{\theta}}_r^t (H_r(\bm{W}^{*})+{\delta}_1)}\mathbb{P}(H_{At}) \leq I \leq \mathrm{e}^{-{\sum}_{r=1}^R {\tilde{\theta}}_r^t (H_r(\bm{W}^{*})-{\delta}_1)}.
\end{equation}
On $H_{At}^c$, with H\"{o}lder inequality,
\begin{equation}
\begin{split}
II &\leq \mathbb{E}_{\bm{W}^{(t)}}[\mathrm{e}^{-2{\sum}_{r=1}^R {\tilde{\theta}}_r^t H_r(\bm{W}^{(t)})}]^{\frac{1}{2}}(\mathbb{P}(H_{At}^c))^{\frac{1}{2}}
\\& \leq \mathbb{E}_{\bm{W}^{(t)}}[\mathrm{e}^{-2{\sum}_{r=1}^R {\tilde{\theta}}_r^t H_r(\bm{W}^{(t)})}]^{\frac{1}{2}}(\mathbb{P}(L(\bm{W}^{(t)})\geq {\delta}_2))^{\frac{1}{2}}.
\end{split}
\end{equation}
According to Lemma 2.1, ${\lim}_{t \rightarrow \infty} \mathbb{P}(L(\bm{W}^{(t)})\geq {\delta}_2) =0$ and $\lim_{t \rightarrow \infty}\mathbb{P}(H_{At})=1$, combining with (\ref{eq:neqI}), taking the limit superior and limit inferior, we have
\begin{equation}
\begin{split}
{\limsup}_{t \rightarrow \infty} \mathrm{e}^{-{\sum}_{r=1}^R {\tilde{\theta}}_r^t (H_r(\bm{W}^{*})+{\delta}_1)}\mathbb{P}(H_{At}) \leq {\Lambda}_* \leq {\liminf}_{t \rightarrow \infty} \mathrm{e}^{-{\sum}_{r=1}^R {\tilde{\theta}}_r^t (H_r(\bm{W}^{*})-{\delta}_1)}.
\end{split}
\end{equation}
Then
\begin{equation}
{\liminf}_{t \rightarrow \infty} {\sum}_{r=1}^R {\tilde{\theta}}_r^t (H_r(\bm{W}^{*})+{\delta}_1) \geq -\ln {\Lambda}_* \geq {\limsup}_{t \rightarrow \infty} {\sum}_{r=1}^R {\tilde{\theta}}_r^t (H_r(\bm{W}^{*})-{\delta}_1).
\end{equation}
Let ${\delta}_1 \rightarrow 0$ we have 
\begin{equation}
\sum_{r=1}^R ({\limsup}_{t \rightarrow \infty}{\tilde{\theta}}_r^t-{\liminf}_{t \rightarrow \infty}{\tilde{\theta}}_r^t)H_r(\bm{W}^{*}) \leq 0,
\end{equation}
since $H_r(\bm{W}^{*})>0, \forall r=1,2,\cdots,R$,
\begin{equation}
{\limsup}_{t \rightarrow \infty}{\tilde{\theta}}_r^t \leq {\liminf}_{t \rightarrow \infty}{\tilde{\theta}}_r^t,\ \forall r=1,2,\cdots,R.
\end{equation}
Therefore we have proved the existence of limits $\tilde{\theta}_r^*, r=1,2,\cdots,R$. 

We proceed to prove that $\forall r,\ \tilde{\theta}_r^*={\theta}_r^*.$ We use the assumption on invariance of the connected correlation functions. First we demonstrate the well-posedness of connected correlation function (\ref{eq:noncanocorre}). From the existence of $\tilde{\theta}_r^*$, both of $\hat{\mathcal{Z}}_*(\{{\mu}_{jk}\},\{{\tilde{\theta}}_r^*\})$ and $f_*(\{{\phi}_{jk}\},\{{\tilde{\theta}}_r^*\})$ exists. As mentioned above, $\hat{\mathcal{Z}}_*(\{{\mu}_{jk}\},\{{\tilde{\theta}}_r^*\})$ is the Laplace transform of $f_*(\{{\phi}_{jk}\},\{{\tilde{\theta}}_r^*\})$.

Let $\sum_{j=1}^M {\phi}_{jk}^2 \triangleq \rho_k, k=1,2,\cdots,S$, $V_k(\vec{\phi}_k)=\sum_{r=1}^R {\tilde{\theta}}_r^* \rho_k^r$, ${\mu}_{jk}=a_{jk} + \mathrm{i} b_{jk}$, we have
\begin{equation}
V_k(\vec{\phi}_k) \geq \frac{1}{2}(\tilde{\theta}_1^*\rho_k +\tilde{\theta}_R^*\rho_k^R),
\end{equation}
let $A \triangleq \max_{jk} |a_{jk}|$, then $-\sum_{j=1}^M a_{jk}{\phi}_{jk} \leq \sum_{j=1}^M |a_{jk}|{\phi}_{jk} \leq A\sqrt{M\rho_k}$ and
\begin{equation}
-V_k(\vec{\phi}_k)-\sum_{j=1}^M a_{jk}{\phi}_{jk} \leq -\frac{1}{2}(\tilde{\theta}_1^*\rho_k -\tilde{\theta}_R^*\rho_k^R)+A\sqrt{M\rho_k}.
\end{equation}
We notice that $\exists \hat{\rho}_k>0$, when$ {\rho}_k \geq \hat{\rho}_k $, there is $-V_k(\vec{\phi}_k)-\sum_{j=1}^M a_{jk}{\phi}_{jk} \leq -\frac{1}{4}\tilde{\theta}_R^*\rho_k^R $. Let $D_k \triangleq \{\vec{\phi}_k:\ {\rho}_k < \hat{\rho}_k\}$, then 
\begin{equation}
\begin{split}
&\int_{\mathbb{R}^{MS}_+} |f_*(\{{\phi}_{jk}\},\{{\tilde{\theta}}_r^*\}) \mathrm{e}^{-\sum_{j,k} \mu_{jk}{\phi}_{jk}}| \Pi_{jk}d\phi_{jk}=\int_{\mathbb{R}^{MS}_+} f_*(\{{\phi}_{jk}\},\{{\tilde{\theta}}_r^*\}) \mathrm{e}^{-\sum_{j,k} a_{jk}{\phi}_{jk}} \Pi_{jk}d\phi_{jk}\\&=\mathrm{e}^{Ng_*} \prod_{k=1}^S (\int_{D_k} \mathrm{e}^{-V_k(\vec{\phi}_k)-\sum_{j=1}^M a_{jk}{\phi}_{jk}}\Pi_{jk}d\phi_{jk} +\int_{D_k^c} \mathrm{e}^{-\frac{1}{4}\tilde{\theta}_R^*\rho_k^R} \Pi_{jk}d\phi_{jk}).
\end{split}
\end{equation}
Since $\int_{D_k} \mathrm{e}^{-V_k(\vec{\phi}_k)-\sum_{j=1}^M a_{jk}{\phi}_{jk}}\Pi_{jk}d\phi_{jk} < \infty $ and ($C$ is a constant):
\begin{equation}
\int_{D_k^c} \mathrm{e}^{-\frac{1}{4}\tilde{\theta}_R^*\rho_k^R} \Pi_{jk}d\phi_{jk} \leq C\int_{u=\sqrt{\hat{\rho}_k}}^{\infty} \mathrm{e}^{-\frac{1}{4}\tilde{\theta}_R^*u^{2R}}u^{M-1} du < \infty,
\end{equation}
the integral in $\hat{\mathcal{Z}}_*(\{{\mu}_{jk}\},\{{\tilde{\theta}}_r^*\})$ is absolutely convergent. Then using the dominated convergence theorem we have
\begin{equation}
\bar{\partial}\hat{\mathcal{Z}}_*(\{{\mu}_{jk}\},\{{\tilde{\theta}}_r^*\})=0,
\end{equation}
where $\bar{\partial}\hat{\mathcal{Z}}=\sum_{j,k}\frac{\partial \hat{\mathcal{Z}}}{\partial \bar{\mu}_{jk}}d\bar{\mu}_{jk}$. According to Def 2.1.1 in \cite{ref60}, we know that function $\hat{\mathcal{Z}}_*(\{{\mu}_{jk}\},\{{\tilde{\theta}}_r^*\})$ of $\mu_{j,k}$ is analytic on the whole space ${\mathbb{C}}^{MS}$, i.e. is an entire function. Since $\hat{\mathcal{Z}}_*(\{{\mu}_{jk}\},\{{\tilde{\theta}}_r^*\})$ is nonvanishing, according to \cite{ref63} Lemma 6.1.10, its logarithm $\ln \hat{\mathcal{Z}}_*(\{{\mu}_{jk}\},\{{\tilde{\theta}}_r^*\})$ can be defined by the principal branch and is holomorphic. This means that as the generating function of the connected correlation functions, the following series converges and its radius of convergence is $\infty$ ($\vec{\alpha}$ is the multi-indices in $\mathbb{N}_0^{MS}$):
\begin{equation}
\ln \hat{\mathcal{Z}}_*(\{{\mu}_{jk}\},\{{\tilde{\theta}}_r^*\})=\sum_{\vec{\alpha}} \tilde{C}_{\vec{\alpha}} {\vec{\mu}}^{\vec{\alpha}} = \sum_{\vec{\alpha}} \frac{1}{\vec{\alpha}!}({\partial}^{\vec{\alpha}}\ln \hat{\mathcal{Z}}_*(\{{\mu}_{jk}\},\{{\tilde{\theta}}_r^*\}))|_{\vec{\mu}=0} {\vec{\mu}}^{\vec{\alpha}}.
\end{equation}
Using Cauchy integral formula we have:
\begin{equation}
\tilde{C}_{\vec{\alpha}}=\frac{1}{(2\pi\mathrm{i})^{MS}} \oint_{|\mu_{11}|=r_{11}}\cdots\oint_{|\mu_{MS}|=r_{MS}} \frac{\ln \hat{\mathcal{Z}}_*(\{{\mu}_{jk}\},\{{\tilde{\theta}}_r^*\})}{{\mu}_{11}^{{\alpha}_{11}+1}\cdots {\mu}_{MS}^{{\alpha}_{MS}+1}}d \vec{\mu}
\end{equation}
for radii $r_{jk}>0,\ j=1,2,\cdots,M,\ k=1,2,\cdots,S$. Therefore, the connected correlation functions for non-canonical RG are well-defined.

If we take the limit $t \rightarrow \infty$ in the condition of invariance of partition function:
\begin{equation}
\mathcal{Z} \equiv \mathcal{Z}_t = \lim_{t \rightarrow \infty}\mathcal{Z}_t =\mathrm{e}^{Ng_*-\sum_{r=1}^R \sum_{k=1}^S {\tilde{\theta}}_r^* (\sum_{j=1}^M \phi_j(x_k;\bm{W}^{*})^2)^r} \triangleq \mathcal{Z}_*,
\end{equation}
substituting the value of $g_*$, it implies
\begin{equation}
\mathrm{e}^{-\sum_{r=1}^R \sum_{k=1}^S {\tilde{\theta}}_r^* (\sum_{j=1}^M \phi_j(x_k;\bm{W}^{*})^2)^r}=\mathrm{e}^{-\sum_{r=1}^R \sum_{k=1}^S \theta_r^* (\sum_{j=1}^M \phi_j(x_k;\bm{W}^{*})^2)^r},
\end{equation}
also when $t=0$ we have $\mathrm{e}^{-\sum_{r=1}^R \sum_{k=1}^S \theta_r(\sum_{j=1}^M \phi_j(x_k;\bm{W}^{(0)})^2)^r}= \mathrm{e}^{Ng_*-\sum_{r=1}^R \sum_{k=1}^S {\tilde{\theta}}_r^* (\sum_{j=1}^M \phi_j(x_k;\bm{W}^{*})^2)^r}$ since $g(0)=0$. Together with assumption (\ref{eq:connectassump}) we have the concrete form of invariance of the connected correlation functions:
\begin{equation}
C_{\vec{\alpha}}\triangleq\frac{1}{\vec{\alpha}!}({\partial}^{\vec{\alpha}}\ln \hat{\mathcal{Z}}_*(\{{\mu}_{jk}\},\{{\theta}_r^*\}))|_{\vec{\mu}=0}=\tilde{C}_{\vec{\alpha}},\ \forall \vec{\alpha}.
\end{equation}
Furthermore, the logarithm of Laplace transform at parameter value $\{{\theta}_r^*\}$ is also well-defined and satisfies
\begin{equation}  \label{eq:laplaceeq}
\ln \hat{\mathcal{Z}}_*(\{{\mu}_{jk}\},\{{\theta}_r^*\})=\ln \hat{\mathcal{Z}}_*(\{{\mu}_{jk}\},\{{\tilde{\theta}}_r^*\}).
\end{equation}
$\hat{\mathcal{Z}}_*(\{{\mu}_{jk}\},\{{\theta}_r^*\})$ is also an entire function and the series $\sum_{\vec{\alpha}} C_{\vec{\alpha}} {\vec{\mu}}^{\vec{\alpha}}$ has the same domain of convergence with $\sum_{\vec{\alpha}} \tilde{C}_{\vec{\alpha}} {\vec{\mu}}^{\vec{\alpha}}$. Taking the exponential of both sides of the equation yields $\hat{\mathcal{Z}}_*(\{{\mu}_{jk}\},\{{\theta}_r^*\})=\hat{\mathcal{Z}}_*(\{{\mu}_{jk}\},\{{\tilde{\theta}}_r^*\}$.

A way to inverse of Laplace transform was introduced by Post and Widder, which can be seen in Theorem.6a of \cite{ref61}. The multi-dimensional form of the in verse formula can be found in \cite{ref62} equation (28). For a function $h \in L^1(\mathbb{R}^n_+)$, $\hat{h}(\mu)$ is its Laplace transform, then
\begin{equation}
h(x)=\frac{1}{(2\pi{\mathrm{i}})^n}\int_{Re {\mu}_1=c_1}\int_{Re {\mu}_2=c_2}\cdots\int_{Re {\mu}_n=c_n}\mathrm{e}^{<\mu,x>}\hat{h}(\mu)d{\mu}_1d{\mu}_2\cdots d{\mu}_n,
\end{equation}
where $c_i$ are real constants s.t. $\hat{h}(\mu)$ is analytic for $Re \mu_i>c_i$. Then from (\ref{eq:laplaceeq}) in addition to $ \hat{\mathcal{Z}}_*(\{{\mu}_{jk}\},\{{\theta}_r^*\})$ and $ \hat{\mathcal{Z}}_*(\{{\mu}_{jk}\},\{{\tilde{\theta}}_r^*\})$ are the Laplace transforms of $f_*(\{{\phi}_{jk}\},\{{\theta}_r^*\})$ and \\$f_*(\{{\phi}_{jk}\},\{{\tilde{\theta}}_r^*\})$, respectively, fix $c_1,c_2,\cdots,c_{MS} \in \mathbb{R}$, we have ($<\mu, \phi>=\sum_{j=1}^M \sum_{k=1}^S \mu_{jk}{\phi}_{jk}$):
\begin{equation}
\begin{split}
f_*(\{{\phi}_{jk}\},\{{\tilde{\theta}}_r^*\})&=\frac{1}{(2\pi{\mathrm{i}})^{MS}}\int_{Re {\mu}_1=c_1}\cdots\int_{Re {\mu}_{MS}=c_{MS}}\mathrm{e}^{<\mu, \phi>} \hat{\mathcal{Z}}_*(\{{\mu}_{jk}\},\{{\tilde{\theta}}_r^*\})d{\mu}_1\cdots d{\mu}_{MS}\\&=\frac{1}{(2\pi{\mathrm{i}})^{MS}}\int_{Re {\mu}_1=c_1}\cdots\int_{Re {\mu}_{MS}=c_{MS}}\mathrm{e}^{<\mu, \phi>} \hat{\mathcal{Z}}_*(\{{\mu}_{jk}\},\{\theta_r^*\})d{\mu}_1\cdots d{\mu}_{MS}.
\end{split}
\end{equation}
This implies that
\begin{equation}
f_*(\{{\phi}_{jk}\},\{{\tilde{\theta}}_r^*\})=f_*(\{{\phi}_{jk}\},\{{\theta}_r^*\}),\ \text{a.e. on}\ {\mathbb{R}}_+^{MS},
\end{equation}
therefore
\begin{equation}
-{\sum}_{k=1}^S {\sum}_{r=1}^R({\tilde{\theta}}_r^*-\theta_r^*)({\sum}_{j=1}^M{\phi}_{jk}^2)^r=0,\ \text{a.e. on}\ {\mathbb{R}}_+^{MS}.
\end{equation}
For this polynomial there must be 
\begin{equation}
{\tilde{\theta}}_r^*=\theta_r^*,\ r=1,2,\cdots,R
\end{equation}
and this is the conclusion of Theorem 3.1. $\square$

According to Theorem 3.1, for data obeying an exponential family distribution as in \eqref{eq:expfam}, after a FCDNN is trained to achieve its unique optimal point, the feature layer output can be viewed as a physical system whose Hamiltonian parameters are equal to the fixed points of the corresponding parameters under the RG for data. Therefore, for this family of data, the FCDNN can extract the same macroscopic features as the RG method, indicating that the reason DNNs can effectively extract features is that they are essentially equivalent to the RG. Moreover, from the proof of Theorem 3.1, it can be seen that this conclusion is actually independent of the specific structure of the network, so corresponding conclusions hold for DNNs with structures beyond FC layers.

\section{Conclusion and future works}

We have established a novel corresponding framework that interprets the training process of neural network models as renormalization transformations, described them using the language of physics, and subsequently provided a mathematical proof that under such correspondence, neural networks can extract the same macroscopic features as those obtained through the RG method. Building upon our previously demonstrated results for the Ising model, this paper proves that for a class of data from exponential family distributions with even-degree polynomial forms, the simulated annealing training process of a FCDNN can be correlated with the scale transformation RG of continuous fields.

By formulating both the input and feature layer output of the neural network in field-theoretic terms, we show that when the network parameters are optimized through training, the feature layer output (viewed as a physical system) exhibits parameter limits that coincide with the fixed points of the corresponding parameters under the RG applied to the data. Consequently, the neural network model can effectively extract macroscopic features from the data in this context.

From an interpretability perspective, the ability of neural networks to extract the main features from data stems from their equivalence to the RG method in the data domain. Under the corresponding Hamiltonian formulation, the features extracted by the neural network align with those derived from the RG method. Therefore, as an exploratory direction for theoretical research into the interpretability of DNNs, comparing them with the RG method is both intuitively sound and practically feasible.

Since we have established a rigorous theoretical framework that maps the training process of DNNs to the RG method, and have completed proofs for the one-dimensional Ising model on the finite lattice and a specific class of exponential family data discussed in this paper, we can extend this framework in the future to a wider range of data distributions, network structures, and training algorithms. This will yield interpretability conclusions that are more closely aligned with practical applications. Specific future directions include the following:

(1) In the current work, regarding the canonical RG for the input data, we assumed the existence of its fixed point. Further efforts can be made to perform computations that rigorously formalize this aspect.

(2) In this paper, we considered data obeying a Hamiltonian of the Landau-Ginzburg-Wilson form. Real-world data distributions may not necessarily satisfy this condition; thus, our conclusions can be further generalized.

(3) By studying the properties of the renormalization flow at the fixed point, we can expect to obtain explicit expressions or equations for the macroscopic features extracted by the network. These could then be experimentally validated on real-world datasets.


\begin{thebibliography}{aa}

\bibitem{ref13}Xia, Z. G., Gong, F. Z., Interpreting deep learning by establishing a rigorous corresponding relationship with renormalization group on Ising model, \textit{Sci. China Math.}, 2026, \textbf{69}(3), 793-812.

\bibitem{ref14}Gong, F. Z., Xia, Z. G., Interpreting Deep Learning by Establishing a Rigorous Corresponding Relationship with Renormalization Group, 2022, arXiv:2212.00005.

\bibitem{ref54}Zinn-Justin, J., Phase transitions and renormalization group, Oxford University Press, Oxford, 2007.

\bibitem{ref53}Pathria, R. K., Statistical Mechanics, Elsevier, Amsterdam, 2016.

\bibitem{ref28}OpenAI, Language models can explain neurons in language models, https://openai.com/index/language-models-can-explain-neurons-in-language-models/, 2023-05-09.

\bibitem{ref29}Anthropic, Mapping the Mind of a Large Language Model, https://www.anthropic.com/research/mapping-mind-language-model, 2024-05-21.

\bibitem{ref30}Guo, D. Y., Yang, D. J., Zhang, H. W., et al. Deepseek-r1 incentivizes reasoning in llms through reinforcement learning. \textit{Nature}, \textbf{645}(8081), 2025,  633-638.

\bibitem{ref31}OpenAI, Detecting misbehavior in frontier reasoning models, https://openai.com/index/chain-of-thought-monitoring/, 2025-03-10.

\bibitem{ref58}Wang, H. Z., Xu, Q. X., Liu, C., et al., Emergent hierarchical reasoning in llms through reinforcement learning, 2025, arXiv:2509.03646.

\bibitem{ref59}Tian, Y. D., Provable Scaling Laws of Feature Emergence from Learning Dynamics of Grokking, 2025, arXiv: 2509.21519.

\bibitem{ref34}Mehta, P., Schwab, D. J., An exact mapping between the variational renormalization group and deep learning, 2014, arXiv:1410.3831.

\bibitem{ref49}Stephan, M., Hoffman, M. D., Blei, D. M., Stochastic gradient descent as approximate Bayesian inference, \textit{J. Mach. Learn. Res.}, \textbf{18}, 2017, 1-35.

\bibitem{ref50}Holley, R. A., Kusuoka, S., Stroock, D. W., Asymptotics of the spectral gap with applications to the theory of simulated annealing, \textit{J. Funct. Anal.},  \textbf{83}, 1989, 333-347.

\bibitem{ref51}Holley, R. A., Stroock, D., Simulated annealing via sobolev inequalities, \textit{Commun. Math. Phys.}, \textbf{115}, 1988, 553-569.

\bibitem{ref55}Ma, S. K., Introduction to the renormalization group, \textit{Rev. Mod. Phys.}, \textbf{45}(4), 1973, 589.

\bibitem{ref56}Morandi, G., Napoli, F., Ercolessi, E., Statistical mechanics: an intermediate course, World Scientific, Singapore, 2001.

\bibitem{ref57}Cui K. Y., Gong F. Z., The Behavior of Observables in Renormalization, \textit{Acta. Math. Sin.-English Ser.}, 2026,  https://doi.org/10.1007/s10114-026-4322-7.

\bibitem{ref60}Hormander, L., An introduction to complex analysis in several variables, Elsevier, Amsterdam, 1973.

\bibitem{ref63}Krantz S. G., Function theory of several complex variables. American Mathematical Soc., Providence, 2001.

\bibitem{ref61}Widder, D. V., The Laplace transform, Princeton University Press, Princeton, 1946.

\bibitem{ref62}Naina Mohammed, S. S., Jeevanandham K., Basherrudin Mahmud Ahmed A., et al., Generalization of multivariable Laplace transform based on Tsallis q-exponential and its inverse using Post-Widder's method, 2022, arXiv:2205.03545.

\end{thebibliography}
\end{document}